\documentclass{IEEEtran}
\usepackage{cite}
\usepackage{amsmath,amssymb,amsfonts}
\usepackage{algorithmic}
\usepackage{graphicx}
\usepackage{textcomp}

\usepackage{url}
\usepackage{subfigure}
\usepackage{amsthm}
\theoremstyle{definition}
\newtheorem{thm}{Theorem}[section]

\def\BibTeX{{\rm B\kern-.05em{\sc i\kern-.025em b}\kern-.08em
    T\kern-.1667em\lower.7ex\hbox{E}\kern-.125emX}}
\begin{document}
\title{Ramp-based Twin Support Vector Clustering}
\author{Zhen Wang, Xu Chen, Chun-Na Li, and Yuan-Hai Shao
\thanks{Submitted in \today. This work is supported in part by National Natural Science Foundation of
China (Nos. 11501310, 61866010, 11871183, and 61703370), in part by Natural Science Foundation of Hainan Province (No. 118QN181), and in part by Scientific Research Foundation of Hainan University (No. kyqd(sk)1804).}
\thanks{Zhen Wang is with
the School of Mathematical Sciences, Inner Mongolia University, Hohhot,
010021 P.R.China (e-mail: wangzhen@imu.edu.cn). }
\thanks{Xu Chen is with
the School of Mathematical Sciences, Inner Mongolia University, Hohhot,
010021 P.R.China (e-mail: pohuozhe@163.com).}
\thanks{Chun-Na Li is with the Zhijiang College,
Zhejiang University of Technology, Hangzhou 310024, P.R.China (e-mail: na1013na@163.com).}
\thanks{Yuan-Hai Shao (*Corresponding Author) is with the School of Economics and Management, Hainan University, Haikou,
570228, P.R.China (e-mail: shaoyuanhai21@163.com).}
}

\maketitle

\begin{abstract}
Traditional plane-based clustering methods measure the cost of within-cluster and between-cluster by quadratic, linear or some other unbounded functions, which may amplify the impact of cost.
This letter introduces a ramp cost function into the plane-based clustering to propose a new clustering method, called ramp-based twin support vector clustering (RampTWSVC). RampTWSVC is more robust because of its boundness, and thus it is more easier to find the intrinsic clusters than other plane-based clustering methods. The non-convex programming problem in RampTWSVC is solved efficiently through an alternating iteration algorithm, and its local solution can be obtained in a finite number of iterations theoretically. In addition, the nonlinear manifold-based formation of RampTWSVC is also proposed by kernel trick. Experimental results on several benchmark datasets show the better performance of our RampTWSVC compared with other plane-based clustering methods.
\end{abstract}


\section{Introduction}
\IEEEPARstart{C}{lustering} that discovers the relationship among data samples, is one of the most fundamental problems in
machine learning
\cite{ClusterBook3,ClusterA4,ClusterA5,ClusterA6}. It has been
applied to many real-world problems, e.g., marketing, text mining,
and web analysis \cite{ClusterA1,ClusterA3}. In particular,
the partition clustering methods \cite{ClusterBook3,ClusterBook5} are widely used in real application for their simplicity, e.g., the classical kmeans \cite{Kmeans} with points as cluster centers, the k-plane clustering (kPC) \cite{Kplane} and proximal-plane clustering (PPC) \cite{PPC,kPPC} with planes as cluster centers. As an extension of point center, the plane center has the ability to discover comprehensive structures in the sample space.

The plane-based clustering seeks the cluster centers depending on the current cluster assignment. When a cluster center is constructed, the similarity of within-cluster is intensified (in some methods, the dissimilarity of between-cluster is also intensified simultaneously). Therefore, the noises or outliers would significantly influence the cluster centers in plane-based clustering. For instance, kPC minimizes the cost of within-cluster by a quadratic function, and PPC minimizes the cost of within-cluster and between-cluster by the same one. Subsequently, the twin support vector clustering (TWSVC) \cite{TWSVC} was proposed, which hired a piecewise linear function to measure the cost of between-cluster but persisted in using the quadratic function for the within-cluster. Recently, another plane-based clustering method, called robust twin support vector clustering (RTWSVC) \cite{RTWSVC}, was proposed by hiring a linear function to measure the cost of within-cluster and between-cluster. Both TWSVC and RTWSVC reduce the influence of noises or outliers to some extent.

\begin{figure}
\centering
    \subfigure[Within-cluster]{\includegraphics[width=0.15\textheight]{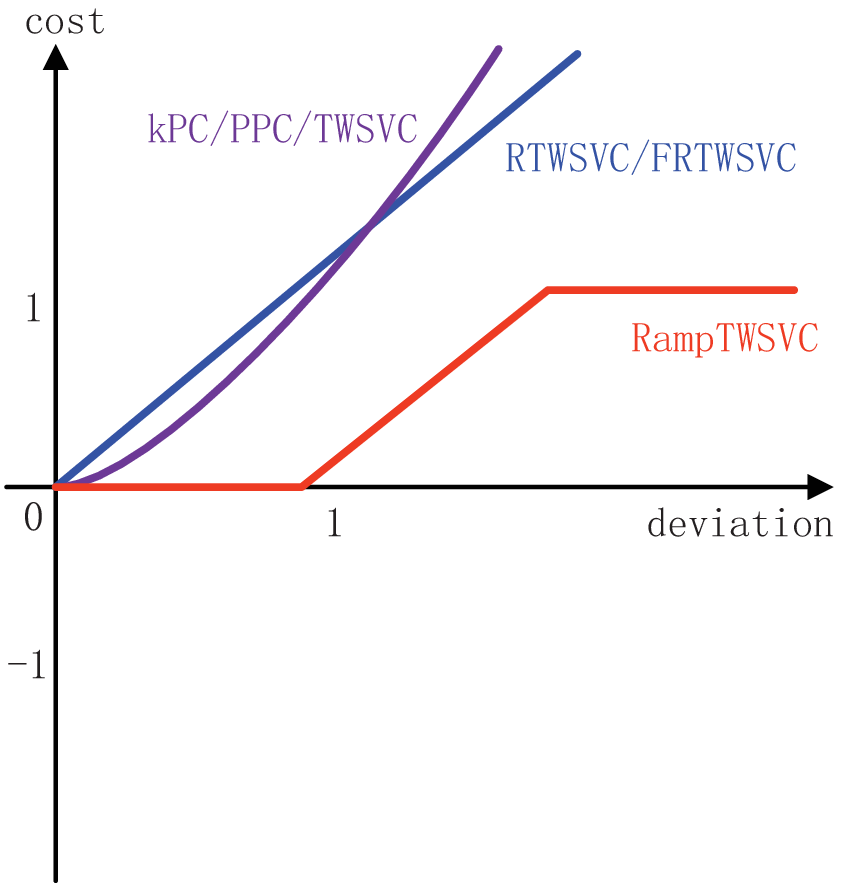}}
    \subfigure[Between-cluster]{\includegraphics[width=0.15\textheight]{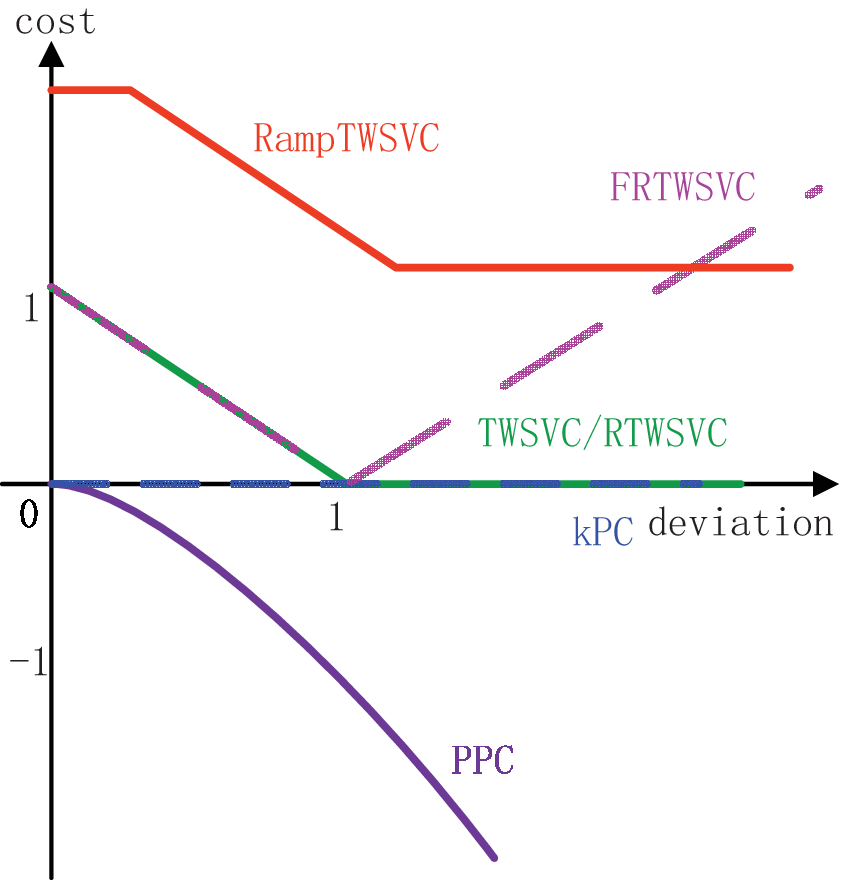}}
\caption{Functions used in kPC, PPC, TWSVC, RTWSVC, FRTWSVC, and RampTWSVC to measure the cost of within-cluster and between-cluster. The horizontal axis denotes the deviation of a sample from the cluster center, and the vertical one denotes the cost to fit the sample. When a cost value is negative, the cost becomes a reward.}\label{Loss}
\end{figure}

The ramp function \cite{RampLoss}, which has been applied in semi-supervised and supervised learning successfully \cite{RampA1,RampA2,RampA3}, is a bounded piecewise linear function. Therefore, in this letter, we propose a ramp-based twin support vector clustering method (RampTWSVC) to further reduce the influence of noises or outliers both from within-cluster and between-cluster, by introducing the ramp function in the construction of the cluster center planes. The problem of RampTWSVC is a non-convex programming problem, and it is recast to a mixed integer programming problem. We propose an iterative algorithm to solve the mixed integer programming problem, and we prove that the algorithm terminates in a finite number of iterations at a local solution. In addition,
RampTWSVC is extended to nonlinear case by kernel trick to cope
with the manifold clustering
\cite{NonlinearClustering1,NonlinearClustering2}.
Fig. \ref{Loss} exhibits the cost functions used in several plane-based clustering methods, where FRTWSVC \cite{RTWSVC} is a plane-based clustering method called fast robust twin support vector clustering. It is obvious from Fig. \ref{Loss} that only our RampTWSVC uses the bounded cost functions both in within-cluster and between-cluster, which can reduce the influence of noises or outliers much more than other methods.
Experimental results on the benchmark datasets show the better performance of the proposed RampTWSVC compared with other plane-based clustering
methods.

%

\section{Review of plane-based clustering}
In this paper, we consider $m$ data samples
$\{x_1,\,x_2,\ldots,x_m\}$ in the $n$-dimensional real vector space
$R^{n}$. Assuming these $m$ samples belong to $k$ clusters
with their corresponding labels in $\{1,\,2,\ldots,k\}$, and they are
represented by a matrix $X=(x_1,x_2,\ldots,x_m)\in R^{n\times m}$.
We further organize the samples from $X$ with the current label $i$ into a matrix
$X_i$ and those with the rest labels into a matrix $\hat{X}_i$,
where $i=1,\,2,\ldots,k$. For readers' convenience, the symbols $X$,
$X_i$, and $\hat{X}_i$ will also refer to the corresponding sets,
depending on the specific context they appear. For example, the
symbol $X$ can be comprehended as a matrix belonging to $R^{n\times m}$ or
a set that contains $m$ samples. The $i$th cluster center plane ($i=1,\ldots,k$) is defined as
\begin{align}\label{Plane}
\begin{array}{l}
f_i(x)=w_i^\top x+b_i=0,
\end{array}
\end{align}
where $w_i\in R^n$ and $b_i\in R$.

The following plane-based clustering methods share the same kmeans-like clustering procedure. Starting from an initial assignment of the $m$ samples into $k$ clusters, all the cluster center planes \eqref{Plane} are constructed by the current cluster assignment. Once obtained all the cluster center planes, the $m$ samples are reassigned by
\begin{align}\label{Predict}
\begin{array}{l}
y=\underset{i}{\arg}\min~|w_i^\top x+b_i|,
\end{array}
\end{align}
where $|\cdot|$ denotes the absolute value. The cluster center
planes and the sample labels are updated alternately until some
terminate conditions are satisfied.
In the following, we briefly describe the different constructions of the cluster center plane by kPC, PPC, TWSVC, and RTWSVC.

\subsection{kPC and PPC}
kPC \cite{Kplane} wishes the cluster center plane close to the current cluster samples. Further on, PPC \cite{PPC} considers it should also be far away from the different cluster samples. Therefore, the $i$th ($i=1,\ldots,k$) cluster center plane in PPC is constructed by solving the following problem
\begin{align}\label{PPC}
\begin{array}{l}
\underset{w_i,b_i}{\min}~~
||X_i^\top w_i+b_ie||^2-c||\hat{X}_i^\top w_i+b_ie||^2\\
s.t.~~~~||w_i||^2=1,
\end{array}
\end{align}
where $||\cdot||$ denotes the $L_2$ norm, $e$ is a column vector of ones with appropriate dimension, and $c>0$ is a user-set parameter. The optimization problem in kPC is just of the first term of the objective of \eqref{PPC}.

From the objective of \eqref{PPC}, it is obvious that a sample from the current cluster receives a quadratic cost, and a sample from a different cluster receives a quadratic reward. Therefore, noises or outliers from the current cluster or different clusters will have great impact on the potential cluster center plane.

\subsection{TWSVC and RTWSVC}
PPC may obtain a cluster center plane which is far from the current cluster, because the samples from different clusters receive high rewards when they are far away from the cluster center plane. In contrast, TWSVC \cite{TWSVC} degrades the reward of the samples from different clusters by considering the problem with $i=1,\ldots,k$ as
\begin{align}\label{TWSVC}
\begin{array}{ll}
 \underset{w_{i},b_{i}}{\min} &\frac{1}{2}\|X_i^\top w_{i}+b_{i}e\|^{2}+ce^{\top}(e-|\hat{X}_i^\top w_{i}+b_{i}e|)_+,\\
 \end{array}
\end{align}
where $(\cdot)_+$ replaces the negative part by zeros.

From the second part of \eqref{TWSVC}, it is available that a sample with a deviation in $[0,1)$ has impact on the cluster center plane. Thus, TWSVC is more robust than PPC. However, the issue of current cluster also exists because of the quadratic cost in the first part of \eqref{TWSVC}. Thus, RTWSVC \cite{RTWSVC} was proposed to decreases the influence of current cluster by replacing the $L_2$ norm in \eqref{TWSVC} with $L_1$ norm. RTWSVC inherits the advantage of TWSVC and decreases the requirement from current cluster. However, the cost of RTWSVC from within-cluster is unbounded from Fig. \ref{Loss}. In order to eradicate the influence of noises or outliers, it is reasonable to hire a bounded function for the within-cluster, whose principle is similar to the cost for the between-cluster used in TWSVC.

\section{RampTWSVC}
Similar to the above plane-based clustering methods mentioned in section 2, our RampTWSVC starts with an initial sample labels, then computes each cluster center plane for the current sample labels iteratively, until some terminate conditions are satisfied. In the following, we consider to obtain one of the cluster center planes for the given samples with their labels.
\subsection{Formation}
To obtain the $i$th ($i=1,\ldots,k$) cluster center plane, our RampTWSVC considers the following problem
\begin{align}\label{Tmp}
\begin{array}{l}
\underset{w_i,b_i}{\min}\frac{1}{2}(||w_i||^2+b_i^2)+c_1\sum\limits_{x_j\in X_i}R_{1}(x_j)+c_2\sum\limits_{x_j\in \hat{X}_i}R_{2}({x}_j),
\end{array}
\end{align}
where $c_1>0$ and $c_2>0$ are parameters. $R_{1}(x)$ and $R_{2}(x)$ are two piecewise linear functions w.r.t. the deviation $|f_i(x)|=|w_i^\top x+b_i|$ (see Fig. \ref{Loss}) as
\begin{align}\label{RampLoss1}
R_{1}(x)=\left\{
\begin{array}{ll}
0&\text{if~~}|f_i(x)|\leq 1-\Delta,\\
1-s& \text{if~~}|f_i(x)|\geq2-\Delta-s,\\
|f_i(x)|-1+\Delta& \text{otherwise},
\end{array}\right.
\end{align}
\begin{align}\label{RampLoss2}
R_{2}(x)=\left\{
\begin{array}{ll}
2+2\Delta&\text{if~~}|f_i(x)|\leq-s,\\
1+\Delta-s&\text{if~~}|f_i(x)|\geq1+\Delta,\\
-|f_i(x)|+2+2\Delta-s& \text{otherwise},
\end{array}\right.
\end{align}
where $\Delta\in[0,1)$ and $s\in(-1,0]$ are two parameters to control the function form (typically, we set $\Delta=0.3$ and $s=-0.2$ in this letter).

It is obvious that both of the cost functions $R_1(x)$ (for the current cluster $X_i$) and $R_2(x)$ (for the different clusters $\hat{X}_i$) have bounds for large deviation. Thus, the noises or outliers much further from the cluster center plane do not have greater impact on the cluster center plane when they meet the bound. The above property indicates our RampTWSVC is more robust than RTWSVC.

In the following, we extend the RampTWSVC to nonlinear manifold clustering, and the solutions to the problems in linear and nonlinear RampTWSVC are elaborated in next subsetion.
The plane-based clustering method can be extend to nonlinear manifold clustering easily by the kernel trick \cite{Kernel2,SGTSVM}. By introducing a pre-defined kernel function $K(\cdot,\cdot)$, the plane-based nonlinear clustering seeks $k$ cluster center manifolds in the kernel generated space as
\begin{align}\label{nPlane}
\begin{array}{l}
g_i(x)=K(x,X)^\top w_i+b_i=0,~i=1,\ldots,k.
\end{array}
\end{align}
Then, the nonlinear RampTWSVC considers to introduce the ramp functions into the plane-based nonlinear clustering. By replacing $f_i(x)$ with $g_i(x)$ in \eqref{RampLoss1} and \eqref{RampLoss2}, and substituting them into \eqref{Tmp}, one can easily obtain $k$ optimization problems to construct the cluster center manifolds \eqref{nPlane}.
When we obtain the $k$ cluster centers \eqref{nPlane}, a sample $x$ is assigned to which cluster depending on
\begin{align}\label{nComY}
y=\underset{i}{\arg}\min |K(x,X)^\top w_i+b_i|.
\end{align}
The procedure of the nonlinear case is the same as the linear one, so the details are omitted.

\subsection{Solution}
In this subsection, we study the solution to the problem \eqref{Tmp}. The corresponding problem in nonlinear RampTWSVC is similar to the one in linear case. For convenience, let $u_i=(w_i^\top,b_i)^\top$, $Z_i$ be a matrix whose $j$th column $z_j$ is $x_j$ with an additional feature $1$ (where the corresponding $x_j$ belongs to the $i$th cluster), and $\hat{Z}_i$ be a matrix whose column similar as $z_j$ (where the corresponding $x_j$ does not belongs to the $i$th cluster). Then, the problem \eqref{Tmp} is recast to
\begin{align}\label{RFDPCtmp}
\begin{array}{l}
\underset{u_i}{\min}~~\frac{1}{2}||u_j||^2+c_1e^\top(-1+\Delta -Z_i^\top u_i)_+
+c_1e^\top(-1+\Delta \\+Z_i^\top u_i)_++c_2e^\top(1+\Delta -\hat{Z}_i^\top u_i)_+
+c_2e^\top(1+\Delta\\+\hat{Z}_i^\top u_i)_+-c_1e^\top(s-2+\Delta-Z_i^\top u_i)_+
-c_1e^\top(s-2+\Delta\\+Z_i^\top u_i)_+-c_2e^\top(s-\hat{Z}_i^\top u_i)_+
-c_2e^\top(s+\hat{Z}_i^\top u_i)_+.
\end{array}
\end{align}

It is easy to see that the above problem is a non-convex programming problem because of the concave part $-(\cdot)_+$. By introducing two auxiliary vectors $p_1\in\{-1,0,1\}^{m_i}$ and $p_2\in\{-1,0,1\}^{m-m_i}$ (where $m_i$ is the sample number of the current $i$th cluster), the above problem is equivalent to the following mixed-integer programming problem
\begin{align}\label{PrimalRFDPC}
\begin{array}{l}
\underset{u_i,p_1,p_2}{\min}~~\frac{1}{2}||u_i||^2+c_1e^\top(-1+\Delta-Z_i^\top u_i)_+\\
+c_1e^\top(-1+\Delta+Z_i^\top u_i)_++c_2e^\top(1+\Delta-\hat{Z}_i^\top u_i)_+\\
+c_2e^\top(1+\Delta+\hat{Z}_i^\top u_i)_++c_1p_1^\top Z_i^\top u_i
+c_2p_2^\top \hat{Z}_i^\top u_i\\
s.t.~~p_1(j)=\left\{\begin{array}{ll}
                                    -1&\text{if~~}z_j^\top u_i>2-\Delta-s,\\                             1&\text{if~~}z_j^\top u_i<-2+\Delta+s,\\
                                    0&\text{otherwise},
                                    \end{array}\right.\forall z_j\in Z_i\\
~~~~~~p_2(j)=\left\{\begin{array}{ll}
                                    -1&\text{if~~} z_j^\top u_i>-s,\\
                                    1&\text{if~~}z_j^\top u_i<s,\\
                                    0&\text{otherwise},
                                    \end{array}\right.\forall z_j\in \hat{Z}_i
\end{array}
\end{align}
where $p_1(j)$ and $p_2(j)$ are the corresponding $j$th elements of $p_1$ and $p_2$, respectively.

Here, we propose an alternating iteration algorithm to solve the mixed-integer programming problem \eqref{PrimalRFDPC}. Starting with an initialized $u_i^{(0)}$, it is easy to calculate $p_1^{(0)}$ and $p_2^{(0)}$ by the constraints of \eqref{PrimalRFDPC}. For fixed $p_1^{(t-1)}$ and $p_2^{(t-1)}$ ($t=1,2,\ldots$), the problem \eqref{PrimalRFDPC} becomes to an unconstrained convex problem
and its solution can be obtained by many algorithms easily (e.g., sequential minimal optimization (SMO) \cite{SMO} and fast Newton-Amijio algorithm \cite{STPMSVM}). Once obtained $u_i^{(t)}$, $p_1^{(t)}$ and $p_2^{(t)}$ are updated again. The loop will be continued until the objective of \eqref{PrimalRFDPC} does not decrease any more.
\begin{thm}
The above alternating iteration algorithm to solve \eqref{PrimalRFDPC} terminates in a finite number of iterations at a local optimal point, where a local optimal point of the mixed integer programming problem \eqref{PrimalRFDPC} is defined as the point $(u_i^*,p_1^*,p_2^*)$ if $u_i^*$ is the global solution to the problem \eqref{PrimalRFDPC} with fixed $(p_1^*,p_2^*)$ and vice versa.
\end{thm}
\begin{proof}
From the procedure of the alternating iteration algorithm, it is obvious that the global solutions to the problem \eqref{PrimalRFDPC} with fixed $u_i$ or $(p_1,p_2)$ are obtained in iteration. Since there is a finite number of ways to select $p_1$ and $p_2$, there are two finite numbers $r_1,r_2>0$ such that $(p_1^{(r_1)},p_2^{(r_1)})=(p_1^{(r_2)},p_2^{(r_2)})$. Thus, we have $u_i^{(r_1)}=u_i^{(r_2)}$. That is to say, the objective values are equal in the $r_1$th and $r_2$th iterations. Since $p_1^\top Z_i^\top u_i\leq0$ and $p_2^\top\hat{Z}_i^\top u_i\leq0$ are always holds, the objective value of \eqref{PrimalRFDPC} keeps non-increasing in iteration. Therefore, the objective is invariant after the $r_1$th iteration, and then the algorithm would terminate at the $r_1$th iteration.

Let us consider the point $(u_i^{(r_1)},p_1^{(r_1)},p_2^{(r_1)})$. From the above proof, we have $G(u_i^{(r_1)},p_1^{(r_1)},p_2^{(r_1)})=G(u_i^{(r_1)},p_1^{(r_1+1)},p_2^{(r_1+1)})$, where $G(\cdot)$ is the objective value of \eqref{PrimalRFDPC}. If there are more than one global solution to the problem \eqref{PrimalRFDPC} with fixed $u_i$, we always select the same one for the same $u_i$. Thus, we have $(p_1^{(r_1)},p_2^{(r_1)})=(p_1^{(r_1+1)},p_2^{(r_1+1)})$, which indicates the point $(u_i^{(r_1)},p_1^{(r_1)},p_2^{(r_1)})$ is a local optimal point.
\end{proof}

\begin{table}
\caption{Details of the benchmark datasets}\centering
\tiny
\begin{tabular}{lllll}
\hline\noalign{\smallskip}
&Data&m  &n &k  \\
(a)&Arrhythmia&452&278&13\\
\hline
(b)&Dermatology&366&34&6\\
\hline
(c)&Ecoli&336&7&8\\
\hline
(d)&Glass&214&9&6\\
\hline
(e)&Iris&150&4&3\\
\hline
(f)&Libras&360&90&15\\
\hline
(g)&Seeds&210&7&3\\
\hline
(h)&Wine&178&13&3\\
\hline
(i)&Zoo&101&16&7\\
\hline
(j)&Bupa&345&6&2\\
\hline
(k)&Echocardiogram&131&10&2\\
\hline
(l)&Heartstatlog&270&13&2\\
\hline
(m)&Housevotes&435&16&2\\
\hline
(n)&Ionosphere&351&33&2\\
\hline
(o)&Sonar&208&60&2\\
\hline
(p)&Soybean&47&35&2\\
\hline
(q)&Spect&267&44&2\\
\hline
(r)&Wpbc&198&34&2\\
\hline \noalign{\smallskip}
\end{tabular} \label{Data}\\
*$m$ is the number of samples, $n$ is the one of dimension, and $k$ is the one of classes.
\end{table}

\section{Experimental results}
In this section, we analyze the performance of our RampTWSVC
compared with kmeans \cite{Kmeans}, kPC \cite{Kplane}, PPC \cite{PPC}, TWSVC \cite{TWSVC}, RTWSVC \cite{RTWSVC}, and FRTWSVC \cite{RTWSVC} on several benchmark datasets \cite{UCI}. All the methods were implemented by MATLAB2017 on a PC with an Intel Core Duo processor (double 4.2 GHz) with 16GB RAM.
The parameters $c$ in PPC, TWSVC, RTWSVC, FRTWSVC, and $c_1$, $c_2$ in RampTWSVC were selected from $\{2^i|i=-8,-7,\ldots,7\}$. For nonlinear case, the Gaussian kernel $K(x_1,x_2)=\exp\{-\mu||x_1-x_2||^2\}$ \cite{Kernel2} was used, and its parameter $\mu$ was selected from $\{2^i|i=-10,-9,\ldots,5\}$. The random initialization was used on kmeans, and the nearest neighbor graph (NNG) initialization \cite{TWSVC} was used on other methods.
In the experiments, we used the metric accuracy (AC) \cite{TWSVC} and mutual information (MI) \cite{PLSC} to measure
the performance of these methods.

Table \ref{Data} shows the details of the benchmark datasets. Tables \ref{Accuracy1} and \ref{Accuracy2} exhibit the linear and nonlinear clustering methods on the benchmark datasets, respectively. The highest metrics among these methods on each dataset are in bold. Besides, we also reported the statistics of these methods in the last rows in Tables \ref{Accuracy1} and \ref{Accuracy2}, which is the number of the datasets that each method is the highest one in terms of AC, MI, or both.

From Table \ref{Accuracy1}, it can be seen that our linear RampTWSVC performs better than other linear methods on five datasets in terms of both AC and MI, and it is more accurate than other methods on other five datasets. On the rest eight datasets, our linear RampTWSVC is also competitive with the best one. From Table \ref{Accuracy2}, it is obvious that our nonlinear RampTWSVC has much higher AC and MI over other methods on many datasets.

\begin{table}
\caption{Linear clustering on benchmark datasets}
\tiny
\begin{tabular}{llllllll}
\hline\noalign{\smallskip}
&kmeans&kPC  &PPC  &TWSVC&RTWSVC&FRTWSVC&Ours  \\
Data & AC(\%)  & AC(\%) & AC(\%) &AC(\%) &AC(\%)&AC(\%)&AC(\%) \\
& MI(\%)  & MI(\%) & MI(\%) &MI(\%) &MI(\%)&MI(\%)&MI(\%) \\
\noalign{\smallskip}\hline\noalign{\smallskip}
(a)&65.72$\pm$0.53&32.31&65.20&32.31&32.31&32.31&$\mathbf{79.42}$\\
&$\mathbf{19.55}$$\pm$0.99&5.49&6.70&5.49&5.49&5.49&10.10\\
\hline
(b)&69.76$\pm$0.77&60.50&70.36&71.93&60.50&60.50&$\mathbf{72.67}$\\
&11.47$\pm$2.15&$\mathbf{29.65}$&3.48&27.40&28.95&28.95&24.42\\
\hline
(c)&82.19$\pm$2.68&33.11&66.46&$\mathbf{85.74}$&34.33&34.33&79.42\\
&$\mathbf{56.84}$$\pm$4.42&8.61&9.65&33.43&10.42&10.42&43.35\\
\hline
(d)&65.58$\pm$3.22&57.73&$\mathbf{66.75}$&66.62&57.59&57.40&62.77\\
&$\mathbf{35.76}$$\pm$2.23&22.55&8.54&35.40&17.69&18.20&20.95\\
\hline
(e)&84.57$\pm$6.86&67.54&60.95&91.24&92.67&$\mathbf{94.95}$&86.79\\
&70.47$\pm$9.10&25.41&12.04&85.59&82.31&$\mathbf{86.97}$&71.71\\
\hline
(f)&$\mathbf{90.84}$$\pm$0.41&89.42&87.93&89.97&89.42&89.42&87.11\\
&$\mathbf{57.50}$$\pm$2.28&56.40&15.84&56.40&56.40&56.40&44.47\\
\hline
(g)&$\mathbf{87.35}$$\pm$0.15&71.80&62.39&63.40&72.24&76.16&74.07\\
&$\mathbf{69.77}$$\pm$0.68&42.43&18.33&51.27&43.17&52.09&45.74\\
\hline
(h)&71.06$\pm$1.29&52.73&57.49&66.90&$\mathbf{72.20}$&70.26&69.45\\
&41.97$\pm$1.44&7.33&4.70&35.48&$\mathbf{45.35}$&41.08&35.16\\
\hline
(i)&87.49$\pm$1.96&54.12&84.06&88.83&54.12&54.12&$\mathbf{90.22}$\\
&71.93$\pm$3.15&34.23&55.56&73.33&32.15&32.15&$\mathbf{76.98}$\\
\hline
(j)&50.39$\pm$0.03&50.31&51.13&51.22&53.34&52.10&$\mathbf{55.82}$\\
&0.09$\pm$0.02&0.22&0.23&0.42&3.73&1.86&$\mathbf{7.07}$\\
\hline
(k)&66.41$\pm$7.92&52.81&56.66&56.10&$\mathbf{75.01}$&$\mathbf{75.01}$&71.84\\
&24.79$\pm$17.27&0.54&2.99&36.87&$\mathbf{39.64}$&$\mathbf{39.64}$&35.46\\
\hline
(l)&51.45$\pm$0.07&50.04&50.35&50.81&51.40&51.40&$\mathbf{51.82}$\\
&1.87$\pm$0.07&0.02&0.15&$\mathbf{13.11}$&1.63&1.67&2.40\\
\hline
(m)&78.83$\pm$0.15&63.77&68.77&75.83&71.40&71.40&$\mathbf{79.61}$\\
&48.07$\pm$0.38&34.16&27.27&45.19&39.36&39.36&$\mathbf{50.15}$\\
\hline
(n)&58.89$\pm$0.00&61.76&53.23&53.85&$\mathbf{67.64}$&66.63&61.76\\
&13.12$\pm$0.00&13.00&3.26&21.13&$\mathbf{23.04}$&21.26&12.91\\
\hline
(o)&50.22$\pm$0.18&49.80&49.99&50.43&51.26&50.06&$\mathbf{51.62}$\\
&0.74$\pm$0.28&0.01&0.23&0.01&2.06&0.67&$\mathbf{4.05}$\\
\hline
(p)&93.41$\pm$13.90&91.67&$\mathbf{100.0}$&50.05&91.67&91.67&$\mathbf{100.0}$\\
&86.95$\pm$27.53&78.05&$\mathbf{100.0}$&1.70&78.05&78.05&$\mathbf{100.0}$\\
\hline
(q)&52.97$\pm$0.00&65.86&50.67&65.86&50.88&50.58&$\mathbf{67.17}$\\
&$\mathbf{8.48}$$\pm$0.00&0.51&0.51&0.51&0.35&0.34&1.15\\
\hline
(r)&56.03$\pm$0.00&52.95&57.95&56.03&53.48&57.15&$\mathbf{64.15}$\\
&0.08$\pm$0.00&0.21&0.27&0.05&0.01&$\mathbf{2.95}$&1.33\\
\hline
AC&2&0&2&1&3&2&10\\
MI&6&1&1&1&3&3&5\\
Both & 2  &0   &1   &0  &3&2&5 \\
\hline \noalign{\smallskip}
\end{tabular} \label{Accuracy1}\\
\end{table}

\begin{table}
\caption{Nonlinear clustering on benchmark datasets}
\tiny
\begin{tabular}{llllllll}
\hline\noalign{\smallskip}
 &kmeans&kPC  &PPC  &TWSVC&RTWSVC&FRTWSVC&Ours  \\
Data& AC(\%)  & AC(\%) & AC(\%) &AC(\%) &AC(\%)&AC(\%)&AC(\%) \\
& MI(\%)  & MI(\%) & MI(\%) &MI(\%) &MI(\%)&MI(\%)&MI(\%) \\
\noalign{\smallskip}\hline\noalign{\smallskip}
(a)&47.32$\pm$3.08 &62.17&$\mathbf{64.82}$&46.89&62.17&62.17&62.19\\
&$\mathbf{10.76}$$\pm$1.24  &10.14  &6.51  &9.65  &10.14  &10.14 &8.93     \\
\hline
(b)& 71.66$\pm$1.26&72.60&70.62&72.60&72.60&72.60&$\mathbf{72.90}$\\
 &17.84$\pm$3.67   &18.00  &3.65  &18.00  &18.00      &18.00 &$\mathbf{26.79}$ \\
\hline
(c)&79.93$\pm$1.24 &82.49&69.13&$\mathbf{88.29}$&82.49&82.68&83.01\\
 &49.31$\pm$2.28  &57.79  &16.46  &$\mathbf{62.21}$  &57.79  &57.57      &49.97  \\
\hline
(d)&69.27$\pm$1.45 &69.04&66.82&70.10&69.04&69.04&$\mathbf{70.77}$\\
 & 37.50$\pm$2.09 &$\mathbf{41.42}$  &7.35  &23.42  &$\mathbf{41.42}$  &$\mathbf{41.42}$      &0.2918  \\
\hline
(e)&87.63$\pm$8.09 &91.24&59.47&91.24&91.24&91.24&$\mathbf{94.95}$\\
 &76.26$\pm$9.85  &79.15  &13.93  &79.15  &79.15  &79.15      &$\mathbf{86.23}$  \\
\hline
(f)&$\mathbf{90.60}$$\pm$0.42 &85.67&88.04&90.08&86.38&86.38&89.60\\
 & 54.86$\pm$1.24 &17.95  &17.79  &$\mathbf{56.98}$  &22.28  &22.28    &51.18  \\
\hline
(g)&87.02$\pm$0.77 &78.41&68.48&81.54&79.03&78.41&$\mathbf{87.14}$\\
 & 69.74$\pm$0.55 &58.81  &26.95  & 63.48 &54.07  &58.81      &$\mathbf{69.98}$  \\
\hline
(h)&52.07$\pm$4.07 &60.75&$\mathbf{72.55}$&44.89&60.75&60.75&64.06\\
 & 13.84$\pm$3.04 &20.35  &$\mathbf{41.23}$  &6.12  &20.35  &20.35     & 25.98 \\
\hline
(i)&87.14$\pm$3.39 &90.63&89.52&90.63&90.63&90.63&$\mathbf{91.25}$\\
 & 70.79$\pm$5.39 &77.99  &72.90  &77.99  & 77.99 &77.99     &$\mathbf{79.70}$  \\
\hline
(j)&51.08$\pm$0.35 &51.22&$\mathbf{53.04}$&51.98&51.22&51.22&$\mathbf{53.04}$\\
 & 0.46$\pm$0.42 &0.37  &2.90  &1.60  &0.37  &0.37      &$\mathbf{4.54}$  \\
\hline
(k)&71.14$\pm$0.82 &55.04&56.66&56.66&55.04&55.04&$\mathbf{71.84}$\\
 & $\mathbf{32.41}$$\pm$0.53 &0.85  &2.73  &2.73  &0.85  &0.85     &28.53  \\
\hline
(l)&50.83$\pm$0.41 &53.00&51.54&50.92&53.00&53.00&$\mathbf{54.91}$\\
 & 1.88$\pm$0.54 &3.79  &1.64  &0.81  &3.79  &3.79     &$\mathbf{6.98}$  \\
\hline
(m)&79.79$\pm$0.94 &75.50&75.83&$\mathbf{91.21}$&75.50&75.50&80.68\\
 &46.91$\pm$1.87  &42.09  &46.38  &$\mathbf{72.31}$  &42.09  &42.09     &48.86  \\
\hline
(n)&62.32$\pm$0.00 &59.14&59.89&60.67&59.14&59.14&$\mathbf{82.92}$\\
 & 22.24$\pm$0.00 &23.79  &10.87  &13.60  &23.79  &23.79     &$\mathbf{52.32}$  \\
\hline
(o)&50.16$\pm$0.28 &51.62&52.66&52.22&51.62&51.62&$\mathbf{54.52}$\\
 & 0.39$\pm$0.39 &4.24  &4.08  &5.43  &4.24  &4.24     &$\mathbf{6.64}$  \\
\hline
(p)&100.0$\pm$0.00&	100.0&	100.0&100.0&100.0&100.0&100.0\\
 & 100.0$\pm$0.00 &100.0  &100.0  &100.0  &100.0  &100.0     &100.0  \\
\hline
(q)&60.68$\pm$4.79 &66.73&68.06&68.06&66.73&66.73&$\mathbf{68.98}$\\
 &3.38$\pm$3.72  &0.17  &2.35  &2.35  &0.17  & 0.17    &$\mathbf{17.69}$  \\
\hline
(r)&63.40$\pm$0.52 &63.08&$\mathbf{64.15}$&63.08&63.08&63.08&63.61\\
 &0.58$\pm$0.52  &0.25  &$\mathbf{1.42}$  &0.25  &0.25  &0.25     &0.25  \\
\hline
AC&1&0&4&2&0&0&11\\
MI&2&1&2&3&1&1&9\\
Both & 0  &0   &2   &2  &0 &0&9  \\
\hline \noalign{\smallskip}
\end{tabular} \label{Accuracy2}\\
\end{table}

\section{Conclusions}
A plane-based clustering method (RampTWSVC) has been
proposed with the ramp function. It contains both the linear and nonlinear formations. The
cluster center planes in RampTWSVC are obtained by solving a series of
non-convex problems, and their local solutions are guaranteed by a proposed alternating iteration algorithm in theory. Experimental results on several benchmark datasets
have indicated that our RampTWSVC performs much better than other plane-based clustering methods on many datasets. For practical convenience,
the corresponding RampTWSVC Matlab code has been uploaded upon \url{http://www.optimal-group.org/Resources/Code/RampTWSVC.html}. Future work includes the parameter regulation and efficient solver design for our non-convex problems.

\bibliographystyle{IEEEtran}
\bibliography{FBib}

%

\end{document}